\documentclass[conference]{IEEEtran}
\IEEEoverridecommandlockouts

\usepackage{cite}
\usepackage{amsmath,amssymb,amsfonts}
\usepackage{amsthm}
\usepackage{mathtools}
\usepackage{graphicx}
\usepackage{textcomp}
\usepackage{xcolor}
\usepackage{algorithm}
\usepackage[noend]{algpseudocode}
\usepackage{array}
\newcolumntype{C}[1]{>{\centering\arraybackslash}m{#1}}
\usepackage{subfigure}
\usepackage{multirow}
\usepackage{tikz}
\usetikzlibrary{shapes,arrows,positioning,calc}

\newcommand{\beq}[1]{\begin{equation}\label{#1}}
\newcommand{\eeq}{\end{equation}}

\newcolumntype{P}[1]{>{\centering\arraybackslash}p{#1}}

\newtheorem{remark}{Remark}
\newtheorem{proposition}{Proposition}

\title{%
Diffusion-Guided Feature Selection via Nishimori Temperature:
Noise-Based Spectral Embedding
}
\author{%
\IEEEauthorblockN{Vasiliy S. Usatyuk, Denis A. Sapozhnikov}
\IEEEauthorblockA{R\&D Department\\
T8 LLC\\
Moscow, Russia\\
Email: L@Lcrypto.com, D@Lcrypto.com}
\and
\IEEEauthorblockN{Sergey I. Egorov}
\IEEEauthorblockA{Department of Computer Science\\
South-West State University\\
Kursk, Russia\\
Email: sie58@mail.ru}}

\begin{document}
\maketitle

\begin{abstract}
We introduce Noise-Based Spectral Embedding (NBSE), a diffusion-driven
framework for selecting informative features from high-dimensional data.
Given $M$ objects and $D$ original features, NBSE constructs a sparse
similarity graph on the objects and determines the Nishimori temperature
$\beta_{N}$---the critical inverse temperature at which the Bethe--Hessian
matrix becomes singular.  The associated smallest eigenvector identifies
the dominant mode of a degree-corrected diffusion process whose time scale
is set automatically by statistical physics.  Repeating this construction
for every original feature yields a $D$-dimensional spectral fingerprint.
By transposing the data matrix and applying NBSE to the features
themselves, we obtain a low-dimensional representation of the feature
space that reveals natural groups of redundant or semantically related
dimensions.  Selecting one representative per group provides a principled,
non-greedy dimensionality reduction that preserves classification
performance while discarding up to $70\%$ of the original features.

The Bethe--Hessian operator
$H(\beta)=\mathbf I+\tilde{\mathbf D}(\beta)-\mathbf S(\beta)$ incorporates
a degree-dependent diagonal term
$\tilde D_{ii}(\beta)=\sum_j\sinh^2(\beta W_{ij})$ that is essential for
sparse graphs.  We show that this term yields an intrinsic degree
correction which mitigates hub dominance in the spectrum, and we prove a
perturbation bound establishing that coloured Gaussian noise of local
variance $\sigma_i^{2}$ shifts $\beta_N$ by at most $O(\bar\sigma^{2})$,
with $\bar\sigma=\max_i\sigma_i$, thereby rigorously justifying the
method's robustness.

Experiments on deep-network embeddings (MobileNetV2 and EfficientNet-B4)
demonstrate that NBSE consistently outperforms classical ANOVA $F$-test
selection and random sampling.  On MobileNetV2, spectral selection
exhibits high stability, with accuracy fluctuations bounded by $5\%$
under aggressive feature reduction, whereas original-data methods lose up
to $25\%$.  On EfficientNet-B4, the synergy of Nishimori-temperature
spectral embedding and specialised selection limits accuracy loss to less
than $1\%$ even at $70\%$ feature reduction ($30\%$ retention),
outperforming baseline methods by up to $6.8\%$.
\end{abstract}
\begin{IEEEkeywords}
Diffusion on graphs, Nishimori temperature, Bethe--Hessian,
noise-based spectral embedding, feature selection, MET QC-LDPC graphs,
dimensionality reduction.
\end{IEEEkeywords}

\section{Introduction}
High-dimensional representations---deep-network embeddings, genomic
vectors, TF-IDF matrices---are now standard in modern machine learning.
While such rich descriptors enable state-of-the-art accuracy, they create
three practical difficulties:
\begin{itemize}
    \item \textbf{Overfitting.} Many dimensions capture noise or
          irrelevant variation.
    \item \textbf{Computational burden.} Training and inference scale at
          least linearly with the number of features.
    \item \textbf{Interpretability loss.} Tracing a decision back to the
          original measurements becomes increasingly hard.
\end{itemize}

Classical dimensionality-reduction techniques such as
PCA~\cite{Jolliffe} or manifold learners (t-SNE~\cite{Maaten2008},
UMAP~\cite{McInnes2018}) focus on preserving the geometry of the object
space.  When downstream models are linear, however, many features are
redundant because they convey essentially the same information with
respect to an underlying similarity graph.

A widely used remedy is the ANOVA $F$-test~\cite{Guyon2003}, which scores
each feature independently and discards those with low marginal relevance.
This filter approach ignores inter-feature correlations; consequently it
may remove useful dimensions or retain noisy ones, especially when
features are strongly collinear.

In this paper we propose \textbf{Noise-Based Spectral Embedding (NBSE)}---
a topology-aware, diffusion-guided method that:
\begin{enumerate}
    \item constructs a sparse quasi-cyclic similarity multigraph on the
          objects,
    \item computes the Nishimori temperature $\beta_N$, i.e.\ the critical
          inverse temperature at which the Bethe--Hessian $H(\beta)$
          becomes singular (Eq.~\eqref{eq:nishimori-condition}),
    \item extracts the eigenvector associated with the smallest eigenvalue
          of $H(\beta_N)$, which identifies the dominant mode of a
          degree-corrected diffusion process evaluated at the statistically
          optimal inverse temperature,
    \item repeats this procedure for every original feature, thereby
          obtaining a spectral fingerprint that quantifies per-feature
          participation in graph diffusion, and
    \item embeds features themselves by transposing the data matrix and
          applying NBSE to the columns, enabling \emph{spectral ablation}:
          a principled reduction that respects global correlations without
          greedy search.
\end{enumerate}
Because the optimal inverse temperature is obtained analytically from the
Nishimori condition rather than tuned via cross-validation, NBSE remains
computationally lean: under a quasi-stationarity hypothesis, a single
sparse eigenvalue problem suffices for all features.

The rest of the paper is organised as follows.  Section~II reviews related
work; Section~III introduces the Ising model, the Nishimori temperature,
and the Bethe--Hessian, derives the degree-corrected diffusion
interpretation, and proves a noise-robustness bound;
Section~IV details NBSE and the spectral ablation algorithm;
Section~V presents experimental results; Section~VI discusses why
degree-corrected diffusion provides a superior basis for feature
selection; and Section~VII concludes.

\section{Related Work}
Dimensionality reduction for classification has a long tradition.
Linear methods such as LDA~\cite{Fisher1936} and PCA maximise class
separation or variance, respectively, but do not directly address feature
redundancy.  Non-linear manifold learners (t-SNE, UMAP) preserve local
neighbourhoods of objects yet are unsuitable for selecting a subset of
original features.

Feature-selection techniques fall into three families:
\begin{itemize}
    \item \textbf{Filter methods} score each dimension independently using
          statistical tests such as the ANOVA $F$-test or mutual
          information~\cite{Guyon2003}.
    \item \textbf{Wrapper methods} rely on greedy forward/backward
          selection with a classifier in the loop~\cite{Kohavi1997}.
    \item \textbf{Embedded methods} perform selection implicitly via
          regularisation (Lasso, Elastic Net)~\cite{Zou2005}.
\end{itemize}

Spectral techniques have been employed for clustering~\cite{Ng2002} and
dimensionality reduction through Laplacian eigenmaps~\cite{Belkin2003},
but they rarely target feature selection directly.  Graph-based diffusion
for semi-supervised learning~\cite{Zhu2003} inspired our use of a
diffusion operator to reveal feature importance.

From the statistical-physics perspective, Saade et al.\
\cite{Saade2014} introduced spectral clustering based on the
Bethe--Hessian matrix and showed that its smallest eigenvalue vanishes at
the Nishimori temperature---the phase-transition point of an Ising model
with quenched disorder.  Dall'Amico et al.~\cite{DallAmico2021} unified
this framework for sparse graphs, proving that the Bethe--Hessian
outperforms both adjacency and standard Laplacian methods precisely
because it incorporates a degree-dependent diagonal correction.  This
insight has been leveraged for node classification in sparse
graphs~\cite{Usatyuk2024,Usatyuk2025}, but never for feature selection.

Our method builds on these ideas: we construct a multi-edge quasi-cyclic
LDPC similarity graph (high sparsity, known spectral properties), compute
the Nishimori temperature $\beta_N$ as the root of
$\lambda_{\min}(H(\beta))=0$, and use the associated eigenvector for
diffusion-style feature ranking.  The key innovation is that we apply this
machinery twice---once on objects (to build diffusion fingerprints) and
once on features (to perform spectral ablation)---thereby exploiting global
inter-feature correlations that filter methods miss.

\section{Background}
\subsection{Ising model, Nishimori temperature and Bethe--Hessian}
\label{sec:ising-bethe-hessian}
Let $G=(V,E,\mathbf W)$ be a weighted undirected graph with $|V|=M$ objects
and symmetric adjacency matrix $\mathbf W\in\mathbb R^{M\times M}$.  The Ising
Hamiltonian on $G$ reads
\beq{eq:ising-hamiltonian}
\mathcal H(\boldsymbol s)= - \sum_{i<j}W_{ij}\,s_i\,s_j,
\qquad s_i\in\{-1,+1\}.
\eeq
At inverse temperature $\beta$, the Boltzmann distribution is
\beq{eq:Boltzman-temp}
p_{\beta}(\boldsymbol s)=\frac{e^{-\beta\mathcal H(\boldsymbol s)}}{Z(\beta)}.
\eeq

The Nishimori line in parameter space corresponds to the set of points
where the temperature matches the noise level of a planted spin-glass
model~\cite{Nishimori2001}.  At the special \emph{Nishimori temperature}
$\beta_N$, the model undergoes a phase transition from a paramagnetic
(disordered) phase to an ordered phase in which cluster structure becomes
visible.  This critical point is characterised by the singularity
condition
\beq{eq:nishimori-condition}
\lambda_{\min}\!\bigl(H(\beta_N)\bigr)=0,
\eeq
where $H(\beta)$ denotes the \emph{Bethe--Hessian matrix}~\cite{Saade2014,DallAmico2021}.
Defining effective edge weights $\omega_{ij}(\beta)=\tanh(\beta W_{ij})$,
its entries are
\beq{eq:bethe-hessian}
H_{ij}(\beta)= \delta_{ij}\Bigl(1+\sum_{k\in\partial i}
\frac{\omega_{ik}^{2}(\beta)}{1-\omega_{ik}^{2}(\beta)}\Bigr)
-\frac{\omega_{ij}(\beta)}{1-\omega_{ij}^{2}(\beta)}.
\eeq
Equivalently, using hyperbolic identities,

\begin{align}
H(\beta) &= \mathbf{I} + \tilde{\mathbf{D}}(\beta) - \mathbf{S}(\beta), \\
\tilde{D}_{ii}(\beta) &= \sum_{j} \sinh^{2}(\beta W_{ij}), \\[6pt]
S_{ij}(\beta) &= \begin{cases} 
\sinh(\beta W_{ij})\cosh(\beta W_{ij}) \\ 
\quad = \tfrac{1}{2}\sinh(2\beta W_{ij}), & (i,j) \in E, \\ 
0, & \text{otherwise}. 
\end{cases}
\end{align}

The diagonal matrix $\tilde{\mathbf D}(\beta)$ encodes a local
\emph{effective-degree} term: each entry aggregates the thermal response
$\sinh^{2}(\beta W_{ij})$ of edges incident to node~$i$.  Because
$\tilde D_{ii}(\beta)$ depends on both the inverse temperature and the local
weight distribution, it varies from node to node on heterogeneous graphs.
This built-in degree correction prevents high-degree hub nodes from
dominating the spectrum---a well-known failure mode of both the adjacency
matrix and the standard unnormalised Laplacian on sparse
graphs~\cite{DallAmico2021,Saade2014}.  The eigenvector
$\boldsymbol\psi_{\min}$ associated with the smallest (near-zero)
eigenvalue of $H(\beta_N)$ provides a spectral coordinate that separates
nodes belonging to different communities.

\begin{remark}[Nishimori identities]
On the Nishimori line $\beta=\beta_N$, exact gauge-symmetry identities
hold~\cite{Nishimori2001}.  In particular, thermal averages of local
observables satisfy relations such as
$\langle s_i s_j\rangle_{\beta_N} = \tanh(\beta_N W_{ij})$ and the internal
energy per bond is self-averaging.  These identities underlie the noise-robustness proof in Section~\ref{sec:noise-model}.
\end{remark}

\subsection{Diffusion interpretation}
\label{sec:diffusion-interpretation}
Consider the standard combinatorial Laplacian $L=\mathbf D_g-\mathbf W$,
where $(\mathbf D_g)_{ii}=\sum_j W_{ij}$ is the weighted graph degree, and
its heat kernel
\beq{eq:heat-kernel}
K(t)=\exp(-tL).
\eeq
For any initial signal $\mathbf f(0)$, diffusion evolves as
$\mathbf f(t)=K(t)\,\mathbf f(0)$.  Expanding $K(t)$ in the eigenbasis of
$L$ yields
\begin{equation*}
K(t)=\sum_{k} e^{-t\lambda_k}\mathbf u_k\mathbf u_k^{\top},
\end{equation*}
so that for large~$t$ the mode associated with the smallest non-zero
eigenvalue of $L$ dominates.

The Bethe--Hessian defines its own degree-corrected diffusion family.  At
$\beta=0$, $\tilde{\mathbf D}(0)=\mathbf 0$ and $\mathbf S(0)=\mathbf 0$, so that
\beq{eq:H-at-zero}
H(0) = \mathbf I,
\eeq
which corresponds to the stable paramagnetic state with a uniform restoring
force.  As $\beta$ increases, the off-diagonal term $-\mathbf S(\beta)$
introduces graph structure while the diagonal $\mathbf I+\tilde{\mathbf D}(\beta)$
counteracts it with a node-dependent restoring force.

To make the relationship precise, define the \emph{Bethe--Hessian
Laplacian}
\beq{eq:bh-laplacian}
\mathcal L_{\rm BH}(\beta) \coloneqq \mathbf I - (\mathbf I+\tilde{\mathbf D})^{-1/2}
    \,\mathbf S(\beta)\,(\mathbf I+\tilde{\mathbf D})^{-1/2},
\eeq
which is well defined for all $\beta\ge 0$ because $\mathbf I+\tilde{\mathbf D}$
is strictly positive definite.  Because $H(\beta)$ and $\mathcal L_{\rm BH}(\beta)$
are related by a congruence transformation, they share the same inertia;
in particular, at $\beta_N$ the null space of $\mathcal L_{\rm BH}(\beta_N)$ is
the image under $(\mathbf I+\tilde{\mathbf D})^{1/2}$ of the null space of
$H(\beta_N)$.  Consequently, $\mathcal L_{\rm BH}$ identifies exactly the same
cluster structure.

The operator $\mathcal L_{\rm BH}$ acts as a generalised Laplacian whose
random-walk transition probabilities are reweighted by effective couplings
$S_{ij}(\beta)$ and normalised by local susceptibilities
$(1+\tilde D_{ii})$.  This weighting suppresses edges that exceed the
thermal scale ($\beta W_{ij}\gg 1$), preventing dominant edges from
distorting the diffusion.  In physical terms, \textbf{computing the
Nishimori temperature is equivalent to finding the optimal inverse
diffusion parameter at which the slowest non-trivial degree-corrected mode
becomes statistically distinguishable from noise}.

The following proposition characterises the small-$\beta$ expansion of
$H(\beta)$ and makes explicit why $\tilde{\mathbf D}(\beta)$ cannot be replaced
by a scalar multiple of the identity on heterogeneous graphs.
\begin{proposition}[Small-$\boldsymbol\beta$ expansion]
For $\beta\to 0$, the Bethe--Hessian admits the expansion
\beq{eq:hessian-expansion}
H(\beta)
= \mathbf I - \beta\,\mathbf W
    + \beta^{2}\,
      \operatorname{diag}\!\Bigl(
        {\textstyle\sum_{j}} W_{ij}^{\,2}
      \Bigr)_{i=1}^{M}
    + O(\beta^{3}).
\eeq
\end{proposition}
\begin{IEEEproof}
Since $\sinh(x)=x+x^{3}/6+O(x^{5})$ and $\cosh(x)=1+x^{2}/2+O(x^{4})$,
we have $\sinh^{2}(x)=x^{2}+O(x^{4})$ and
$\sinh(x)\cosh(x)=x+\tfrac{2}{3}x^{3}+O(x^{5})$.
Substituting $x=\beta W_{ij}$ into the definitions of
$\tilde{\mathbf D}(\beta)$ and $\mathbf S(\beta)$ yields the stated expansion.
\end{IEEEproof}

\begin{remark}[Role of the degree-dependent diagonal]
On a \emph{$d$-regular} graph with uniform weights $W_{ij}=w$, every node
has $\sum_j W_{ij}^2 = dw^2$ and the $\beta^{2}$ term in
Eq.~\eqref{eq:hessian-expansion} is proportional to the identity.
However, on sparse real-world graphs---including our QC-LDPC similarity
graphs built from data---effective degrees vary considerably.  The term
$\operatorname{diag}(\sum_j W_{ij}^2)$ then encodes essential local
heterogeneity: high-degree (or strongly weighted) nodes receive larger
diagonal entries, which reduces their influence in the spectrum of
$H(\beta)$ and removes the spectral bias that plagues adjacency-based
methods.  Conflating $\tilde{\mathbf D}(\beta)$ with a scalar multiple of
$\mathbf I$ would destroy this intrinsic degree correction.
\end{remark}

\subsection{Noise model for robustness}
\label{sec:noise-model}
To study measurement noise we perturb each object vector with coloured
Gaussian noise:
\beq{eq:noise-model}
\widetilde{\mathbf p}_i = \mathbf p_i + \boldsymbol\varepsilon_i,
\qquad
\boldsymbol\varepsilon_i\sim\mathcal N\bigl(0,\sigma_i^2\,\mathbf I_{d}\bigr).
\eeq
The local scale $\sigma_i$ is proportional to the typical distance from
node~$i$ to its $k$-nearest neighbours, ensuring that the perturbation is
small relative to inter-cluster gaps.

Adding this noise before graph construction perturbs the edge weights.
Provided the similarity function is smooth and the signal-to-noise ratio
is bounded, the resulting relative change in each weight satisfies a bound
of the form $|\widetilde W_{ij}-W_{ij}|\le \varepsilon |W_{ij}|$ with
$\varepsilon=O(\bar\sigma^{2})$, where $\bar\sigma=\max_i\sigma_i$.
We now prove that, provided this relative perturbation is small and the
smallest eigenvalue curve of $H(\beta)$ is non-degenerate at its root,
the Nishimori temperature shifts only at first order in the weight error,
and hence at second order in the noise amplitude.

\begin{proposition}[Stability of the Nishimori temperature under noise]
\label{prop:noise-stability}
Let $G$ be a sparse graph with bounded average degree, Bethe--Hessian
$H(\beta;\mathbf W)$ and Nishimori temperature $\beta_N$.  Let
$\widetilde{\mathbf W}$ denote the edge weights obtained after perturbing
the data according to Eq.~\eqref{eq:noise-model}, and let
$\widetilde\beta_N$ be the corresponding Nishimori temperature.
Suppose there exists a constant $\varepsilon>0$ such that for every edge
$(i,j)\in E$,
\beq{eq:weight-perturbation-bound}
|\widetilde W_{ij}-W_{ij}|\le \varepsilon\,|W_{ij}|.
\eeq
If the non-degeneracy condition
$g:=\partial_\beta\lambda_{\min}(H(\beta))\big|_{\beta=\beta_N}\neq 0$
holds, then
\beq{eq:beta-shift-bound}
|\widetilde\beta_N-\beta_N|
\le \frac{C_1\,\varepsilon\,\sqrt{|E|}}{|g|}
+ O(\varepsilon^{2})
\eeq
for a constant $C_1$ depending on $\beta_N$ and the weight scale.  In
particular, for coloured noise with local variance $\sigma_i^2$, we have
$\varepsilon=O(\bar\sigma^2)$ where $\bar\sigma=\max_i \sigma_i$.
\end{proposition}
\begin{IEEEproof}
The proof combines three ingredients.

\textbf{(a) Perturbation of the Bethe--Hessian.}
From Eq.~\eqref{eq:bethe-hessian}, the change in $H$ induced by
$\delta W_{ij}=\widetilde W_{ij}-W_{ij}$ has diagonal and off-diagonal
contributions:
\begin{align*}
\delta H_{ii}
&= \sum_j\bigl[\sinh^{2}(\beta(W_{ij}+\delta W_{ij}))
     -\sinh^{2}(\beta W_{ij})\bigr] \\
&= \beta\sum_j\sinh(2\beta W_{ij})
      \,\delta W_{ij}
   + O(\varepsilon^2), \\[4pt]
\delta H_{ij}
&= -\bigl[\sinh(\beta(W_{ij}+\delta W_{ij}))\cosh(\beta(W_{ij}+\delta W_{ij})) \\
&\qquad\; -\sinh(\beta W_{ij})\cosh(\beta W_{ij})\bigr] \\[2pt]
&= -\beta\cosh(2\beta W_{ij})\,\delta W_{ij}
   + O(\varepsilon^2),\qquad i\neq j.
\end{align*}
Because $|\sinh|$ and $|\cosh|$ are bounded on the support of the weights,
both terms are uniformly $O(\varepsilon)$.  For a sparse graph with
bounded average degree, the number of non-zero entries in $H$ is
$\Theta(|E|)$; consequently $\|\delta H\|_F \le C''\varepsilon\sqrt{|E|}$.

\textbf{(b) Weyl's inequality and eigenvalue shift.}
For symmetric matrices, Weyl's inequality gives
\begin{equation*}
|\lambda_k(\widetilde H)-\lambda_k(H)|
\le \|\delta H\|_2
\le \|\delta H\|_F
\le C''\varepsilon\sqrt{|E|}.
\end{equation*}
At $\beta=\beta_N$ we have $\lambda_{\min}(H)=0$, so the perturbed matrix
satisfies $|\lambda_{\min}(\widetilde H)|\le C''\varepsilon\sqrt{|E|}$.

\textbf{(c) Implicit function theorem.}
The Nishimori condition $\lambda_{\min}(H(\beta))=0$ defines a smooth
level set near $\beta_N$ because the non-degeneracy hypothesis guarantees
$g=\partial_\beta\lambda_{\min}\neq 0$.  By the implicit function theorem,
\begin{equation*}
\widetilde\beta_N-\beta_N
= -\,\frac{\delta\lambda_{\min}}
          {g}\Big|_{\beta=\beta_N}
  + O(\varepsilon^2).
\end{equation*}
The numerator is bounded by part~(b) and the denominator equals $|g|$ by
definition.  Combining these estimates yields Eq.~\eqref{eq:beta-shift-bound}.
Finally, for Gaussian noise satisfying Eq.~\eqref{eq:noise-model}, the
relative perturbation of each squared distance is
$O(\sigma_i^2/\|\mathbf p_i-\mathbf p_j\|^2)$; assuming a bounded
signal-to-noise ratio gives $\varepsilon=O(\bar\sigma^2)$.
\end{IEEEproof}

\begin{remark}
The bound in Eq.~\eqref{eq:beta-shift-bound} confirms that noise acts as
a small perturbation on the critical temperature, not as a systematic
bias.  Physically, this is because at $\beta_N$ the system sits at a phase
transition between paramagnetic and ordered regimes; non-degeneracy of
the smallest-eigenvalue curve protects the singular condition against local
weight fluctuations.
\end{remark}

\section{Proposed Method}
\subsection{Similarity graph construction}
All subsequent steps rest on a similarity graph among objects.  Given raw
vectors $\mathbf x_i\in\mathbb R^D$, we define edge weights via the Gaussian
kernel with adaptive local scales:
\beq{eq:edge-weight}
W_{ij}= \exp\!\Bigl(
  -\frac{\|\mathbf x_i-\mathbf x_j\|^2}{\sigma_i\sigma_j}
\Bigr),
\eeq
where the local scale $\sigma_i$ is the average distance to the $k=10$
nearest neighbours of object~$i$.  We instantiate this kernel on a sparse
quasi-cyclic LDPC backbone (Section~\ref{sec:qc-ldpc}).

\subsection{Noise-Based Spectral Embedding (NBSE)}
\label{sec:nbse}
Given a data matrix $\mathbf X\in\mathbb R^{M\times D}$ (rows = objects,
columns = features), we apply the following pipeline for each
feature~$l$, using the univariate observations $\{x_{il}\}_{i=1}^M$ as
scalar node attributes:
\begin{enumerate}
    \item \textbf{Graph construction.}
          Build a sparse similarity graph $G^{(l)}$ on the $M$ objects
          via Eq.~\eqref{eq:edge-weight} with distances computed in the
          one-dimensional feature space.
    \item \textbf{Nishimori temperature.}
          Solve Eq.~\eqref{eq:nishimori-condition} for $\beta_N^{(l)}$.
          This is a one-dimensional root-finding problem; because
          $f(\beta)\coloneqq\lambda_{\min}(H(\beta))$ is smooth and
          monotonically increasing in the relevant regime, bisection or
          Brent's method converges reliably.  The smallest eigenvalue is
          obtained with a few Lanczos iterations exploiting sparsity.
    \item \textbf{Bethe--Hessian eigenvector.}
          Form $H(\beta_N^{(l)})$ via Eq.~\eqref{eq:bethe-hessian} and
          extract its smallest eigenvector $\boldsymbol\psi_{\min}^{(l)}$.
          By Proposition~1, the degree-dependent diagonal correctly
          normalises heterogeneous degrees.
    \item \textbf{One-dimensional embedding.}
          The vector $\boldsymbol\psi_{\min}^{(l)}\in\mathbb R^{M}$ is a
          diffusion-based coordinate for feature~$l$, indicating how that
          feature participates in the slowest degree-corrected flow on its
          graph.
\end{enumerate}
Collecting all coordinates yields the \emph{spectral fingerprint matrix}
\beq{eq:fingerprint-matrix}
\Psi = \bigl[\,\boldsymbol\psi_{\min}^{(1)}\;\boldsymbol\psi_{\min}^{(2)}\;
              \cdots\;\boldsymbol\psi_{\min}^{(D)}\,\bigr]
\in\mathbb R^{M\times D}.
\eeq
Each column of $\Psi$ encodes the diffusion behaviour of one original
feature.  Features that induce similar univariate graph topologies---and
hence similar eigenvectors---are likely to be redundant or semantically
related.

Under the \emph{quasi-stationary hypothesis} we may compute a single
global value $\beta_N$ on an aggregate object graph and reuse it for all
univariate Bethe--Hessians, reducing complexity from $O(D\cdot M\log M)$
to $O(M\log M)$.  Proposition~\ref{prop:noise-stability} justifies this
simplification: sufficiently small inter-feature perturbations shift
$\beta_N$ by at most $O(\varepsilon/|g|)$.

\subsection{Spectral diffusion-guided feature ablation}
\label{sec:spectral-ablation}
Directly clustering the columns of $\Psi$ would require pairwise
comparisons among $D$ vectors.  We avoid this cost by embedding features
themselves: transpose the data matrix to treat features as nodes in a
$D$-node graph and apply NBSE to the resulting feature--feature affinity.

The eigenvector obtained from this transposed diffusion, denoted
$\boldsymbol\phi_{\min}\in\mathbb R^{D}$, orders features according to the
similarity of their roles in the joint diffusion geometry.
Algorithm~\ref{alg:spectral-ablation} implements balanced histogram
binning on $\boldsymbol\phi_{\min}$: its range is divided into $n$ equal
intervals, and one representative per interval is selected.  The resulting
index set $\mathcal I\subset\{1,\dots,D\}$, with $|\mathcal I|=n$,
preserves the global diffusion structure while discarding redundant
dimensions.

The topology of the QC-LDPC graph strongly influences the histogram of
feature embeddings (Fig.~\ref{fig:Feature_frequency_to_frequency_interval}).
A unique Nishimori temperature together with a low-energy eigenvector
produces well-separated clusters in the transposed feature space,
allowing binning ablation to select one representative from each cluster.
By contrast, Erd\H{o}s--R\'enyi graphs do not satisfy the spectral
requirements of diffusion-guided ablation; their spectra fail to reveal
informative feature groups~\cite{Usatyuk2024,Usatyuk2025}.  The same
limitation holds for non-constrained expanders, fully-connected graphs,
and line graphs.

\begin{figure*}[t]
    \centering
    \includegraphics[width=0.48\linewidth]{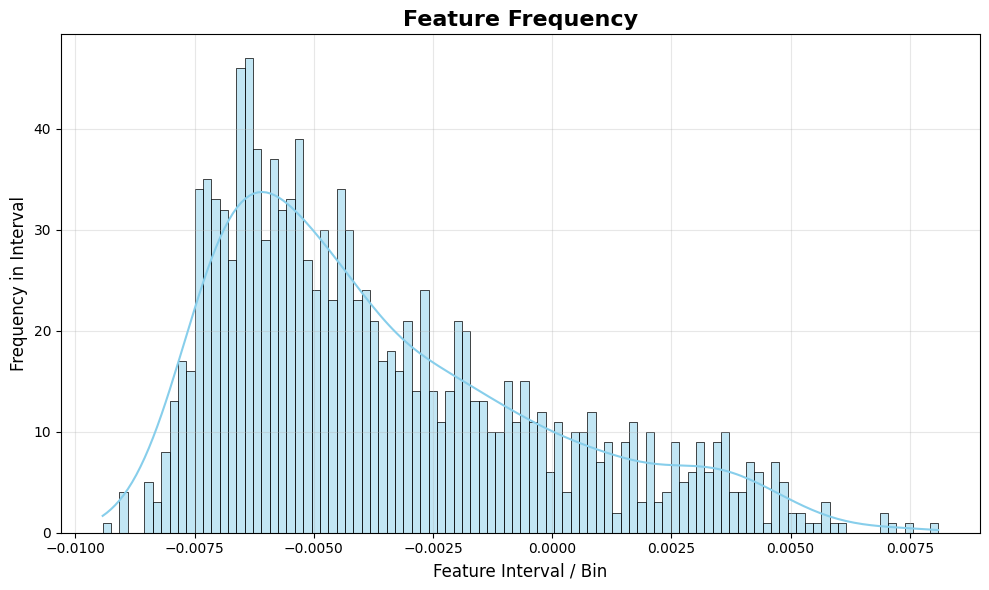}
    \hfill
    \includegraphics[width=0.48\linewidth]{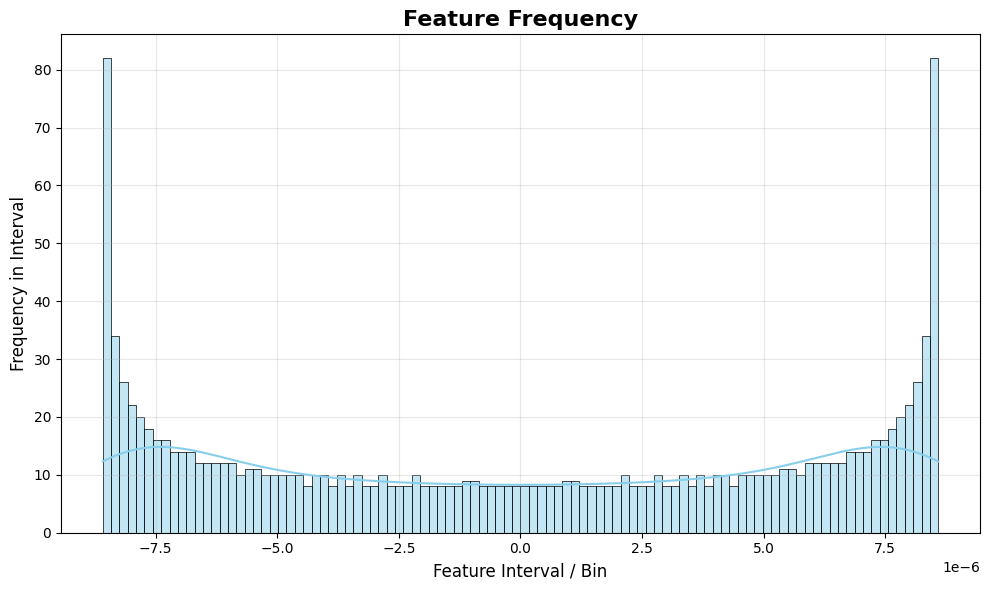}
    \caption{(Left) Bad distribution of features: random placement yields
             no discernible spectral structure.
             (Right) Good multimodal distribution: compact, well-separated
             clusters reflect semantic or statistical similarity.}
    \label{fig:Feature_frequency_to_frequency_interval}
\end{figure*}

\begin{algorithm}[!t]
\caption{Spectral Diffusion-Guided Ablation of Features}%
\label{alg:spectral-ablation}
\begin{algorithmic}[1]
\Require $\mathbf X\in\mathbb R^{M\times D}$, target number $n$ of
         retained features, QC-LDPC graph parameters.
\Ensure Representative index set $\mathcal I$, $|\mathcal I|=n$.
\State // Stage 1: Feature-level spectral embedding
\State $\mathbf Z \gets \mathbf X^{\top}$
       \Comment{Treat features as $D$ objects in $\mathbb R^M$}
\State Build feature affinity matrix $\mathbf C\in\mathbb R^{D\times D}$
       from $\mathbf Z$ via Eq.~\eqref{eq:edge-weight}.
\State Construct $\tilde{\mathbf D}(\beta)$ and $\mathbf S(\beta)$ from
       $\mathbf C$ using Eq.~\eqref{eq:bethe-hessian}.
\State $\beta_N \gets$ root of
       $\lambda_{\min}\bigl(\mathbf I+\tilde{\mathbf D}(\beta)-\mathbf S(\beta)\bigr)=0$
\State $H \gets \mathbf I+\tilde{\mathbf D}(\beta_N)-\mathbf S(\beta_N)$
\State $\boldsymbol\phi_{\min} \gets$ eigenvector for smallest
       eigenvalue of $H$
       \Comment{$\boldsymbol\phi_{\min}\in\mathbb R^{D}$}
\If{$n\le0$}
    \Return $\emptyset$
\ElsIf{$n\ge D$}
    \Return $\{1,\dots,D\}$
\EndIf
\State // Stage 2: Balanced histogram binning
\State $a_{\min}\gets\min_i (\boldsymbol\phi_{\min})_i$,
       $\; a_{\max}\gets\max_i (\boldsymbol\phi_{\min})_i$
\If{$a_{\min}=a_{\max}$}
    \Return first $n$ indices
\EndIf
\State Partition $[a_{\min},\,a_{\max}]$ into $n$ equal intervals
       $\mathcal B_0,\dots,\mathcal B_{n-1}$.
\For{$k=0$ \textbf{to} $n-1$}
    \State $\mathcal C_k \gets \{ i : (\boldsymbol\phi_{\min})_i\in\mathcal B_k \}$
    \State $q_k \gets \bigl\lfloor n\,|\mathcal C_k|/D \bigr\rceil$
           \Comment{Proportional quota, rounded}
\EndFor
\State Adjust quotas $\{q_k\}$ so that $\sum_k q_k = n$
       (add/subtract remainder to/from largest bins).
\State $\mathcal I \gets \emptyset$, \quad $\mathcal U \gets \emptyset$
       \Comment{Used-feature tracker}
\For{$k=0$ \textbf{to} $n-1$}
    \If{$q_k = 0$}\; \textbf{continue} \EndIf
    \State $\mathcal C_k' \gets \mathcal C_k\setminus\mathcal U$
    \If{$\mathcal C_k'=\emptyset$}
        \Comment{Fallback: nearest unused feature to bin centre $m_k$}
        \State $j^\star \gets \arg\min_{j\notin\mathcal U}\,
               |(\boldsymbol\phi_{\min})_j - m_k|$
    \Else
        \State Sort $\mathcal C_k'$ by $(\boldsymbol\phi_{\min})_i$ ascending.
        \State Select $q_k$ indices uniformly from the sorted list
               (including end points).
    \EndIf
    \State Add selected index(indices) to $\mathcal I$ and mark as used.
\EndFor
\State \Return $\mathcal I$
\end{algorithmic}
\end{algorithm}

\subsection{Classification pipeline}
After obtaining the index set $\mathcal I$, we retain the corresponding
columns,
$\mathbf X_{\text{red}} = \mathbf X_{[:,\mathcal I]}$,
and train a linear classifier (logistic regression or linear SVM).  At
test time, raw features are computed, unselected dimensions are discarded,
and the reduced vector is fed to the trained model.  No retraining of the
diffusion operator is required.

The complete inference pipeline consists of five stages:
\begin{enumerate}
    \item Compute object-level NBSE fingerprints $\Psi$
          (Eq.~\eqref{eq:fingerprint-matrix}).
    \item Transpose $\mathbf X$, build a QC-LDPC feature graph, and compute
          its smallest Bethe--Hessian eigenvector
          $\boldsymbol\phi_{\min}$.
    \item Apply Algorithm~\ref{alg:spectral-ablation} to obtain the index
          set $\mathcal I$.
    \item Train a linear classifier on the reduced matrix
          $\mathbf X_{[:,\mathcal I]}$.
    \item At test time, extract features, select columns $\mathcal I$, and
          classify.
\end{enumerate}

\subsection{QC-LDPC graph construction}
\label{sec:qc-ldpc}
The quality of NBSE depends critically on graph topology.  We construct a
\emph{multi-edge type quasi-cyclic low-density parity-check} (MET QC-LDPC)
similarity graph~\cite{METLDPC} with the following properties:
\begin{itemize}
    \item \textbf{Sparsity:} Average degree $d_{\text{avg}}\approx 18$,
          ensuring $O(M)$ edges and fast sparse eigenvalue solvers.
    \item \textbf{Quasi-cyclic structure:} The adjacency matrix is a block
          circulant lift of a small protograph, which guarantees regular
          spectral properties and enables efficient construction via the
          method in~\cite{Usatyuk2025}.
    \item \textbf{Large girth:} Short cycles are avoided (girth $\ge 6$),
          making the graph locally tree-like.  This is essential because
          the Bethe approximation underlying Eq.~\eqref{eq:bethe-hessian}
          is asymptotically exact on trees.
    \item \textbf{Multi-edge type:} Different edge types encode different
          similarity scales, improving the resolution of $\beta_N$
          estimation.
\end{itemize}
These constraints ensure that the graph ensemble has a well-defined degree
distribution and concentrates its spectral properties.  In practice,
$\beta_N$ is stable across random lifts of the same protograph,
contributing to the experimental robustness reported in Section~V.

\section{Experimental Evaluation}
\subsection{Setup}
We used two deep-network encoders as feature generators:
\begin{itemize}
    \item \textbf{MobileNetV2}~\cite{Sandler2018}: $1280$-dimensional
          penultimate layer.
    \item \textbf{EfficientNet-B4}~\cite{Tan2019}: $1792$ dimensions.
\end{itemize}
Both models were pre-trained on ImageNet and evaluated on the ImageNet-1K
validation set ($1000$ classes, $50\,000$ images).  To probe robustness
under varying representation quality, we created two regimes:
\begin{itemize}
    \item \emph{Best representation:} graph with $35\,000$ nodes;
          training classifier on $30\,000$ frozen embeddings.
    \item \emph{Worst representation:} graph with $25\,000$ nodes;
          training classifier on $20\,000$ frozen embeddings.
\end{itemize}

\subsection{Graph construction}
For every dataset we built one quasi-cyclic LDPC similarity graph:
\begin{itemize}
    \item Average degree $d_{\text{avg}}\approx 18$,
    \item Edge weight
          $W_{ij}= \exp(-\|\mathbf x_i-\mathbf x_j\|^2/(\sigma_i\sigma_j))$
          (Eq.~\eqref{eq:edge-weight}),
    \item Local scale $\sigma_i$ computed as the mean distance to the
          $k=10$ nearest neighbours.
\end{itemize}
The graph satisfies the numerical constraints of an $(m,L)$ MET QC-LDPC
code~\cite{METLDPC}, guaranteeing fast eigenvalue computation using the
method in~\cite{Usatyuk2025}.

\subsection{Baselines}
We compared three feature-reduction strategies:
\begin{enumerate}
    \item \textbf{NBSE topology-informed diffusion-guided spectral ablation}
          (Algorithm~\ref{alg:spectral-ablation});
    \item \textbf{ANOVA $F$-test}---select the top-$n$ features with highest
          $F$ scores~\cite{Guyon2003}.  Implemented via scikit-learn
          \cite{Anova};
    \item \textbf{Random choice}---uniformly random selection of $n$
          features.
\end{enumerate}
For each method we varied the retained feature proportion
$p\in\{1.0,0.9,\dots,0.3\}$ and measured classification accuracy of a
logistic-regression model (Crammer--Singer multi-class formulation).  Each
experiment was repeated five times with different random seeds; mean
$\pm$ std is reported.

\subsection{Results}
We evaluated two variants for estimating $\beta_N$:
\begin{enumerate}
    \item \textbf{Global quasi-stationary}---compute a single $\beta_N$
          on the full object graph and reuse it for all features (left
          panels in Figs.~\ref{fig:accuracy-curves1},
          \ref{fig:accuracy-curves2});
    \item \textbf{Independent per-feature}---compute a separate
          $\beta_N^{(l)}$ for each univariate feature slice (right panels).
\end{enumerate}
The global variant simplifies computation and yields virtually identical
classification performance, confirming that the Nishimori temperature is
stable across feature slices under the quasi-stationarity hypothesis.
This empirical observation is consistent with
Proposition~\ref{prop:noise-stability}: if inter-feature perturbations are
small relative to cluster separation, $\beta_N$ shifts negligibly.

Figs.~\ref{fig:accuracy-curves1} (MobileNetV2) and
\ref{fig:accuracy-curves2} (EfficientNet-B4) show classification accuracy
versus retained feature proportion.  The main observations are:

\textbf{MobileNetV2.}
The spectral selection method on the best representation exhibits high
stability under feature reduction, with accuracy fluctuations no greater
than $5\%$; for the worst representation the fluctuation is below $3\%$.
On the spectral embedding, ANOVA and random selection perform more
reliably than on original data.  Random selection unexpectedly outperforms
ANOVA, likely because ANOVA's greedy nature selects overly correlated
features.  Working directly with original data leads to severe degradation:
the model loses up to $17\%$ accuracy with random selection and up to
$25\%$ with ANOVA.

\textbf{EfficientNet-B4.}
With this more powerful architecture, all methods gain stability.
The best result is the combination of spectral embedding and NBSE selector
on a good representation: accuracy loss is less than $1\%$ even at $70\%$
feature reduction ($30\%$ retention).  This outperforms random selection
on spectral embedding by $0.5\%$, ANOVA on spectral embedding by $2.7\%$,
and both original-data methods by a significant $6.8\%$.
These results confirm that the synergy of spectral data representation and
specialised selection preserves information content during aggressive
compression, especially with efficient neural architectures.

\begin{figure*}[t]
    \centering
    \includegraphics[width=0.48\linewidth]{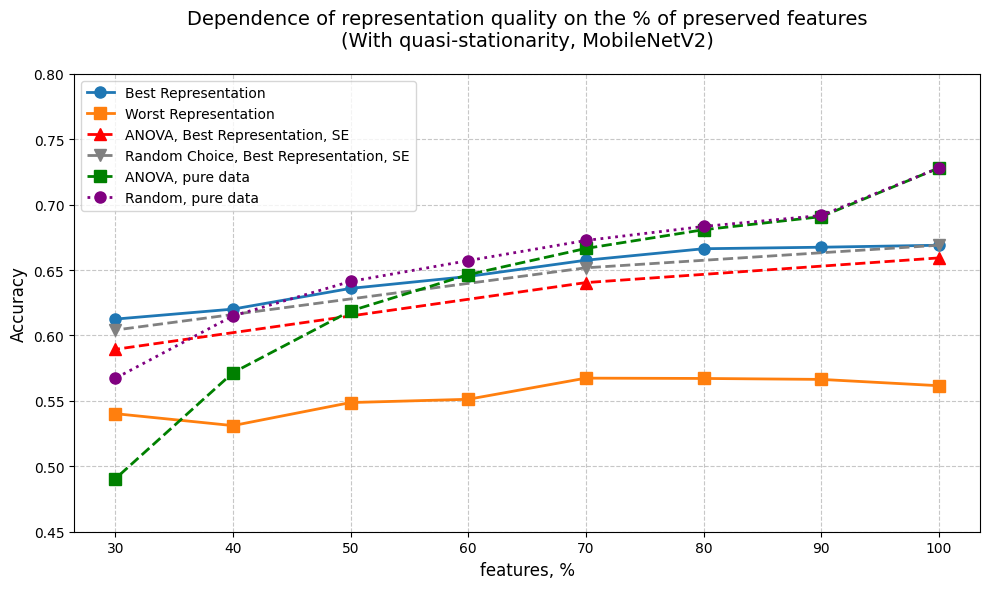}
    \hfill
    \includegraphics[width=0.48\linewidth]{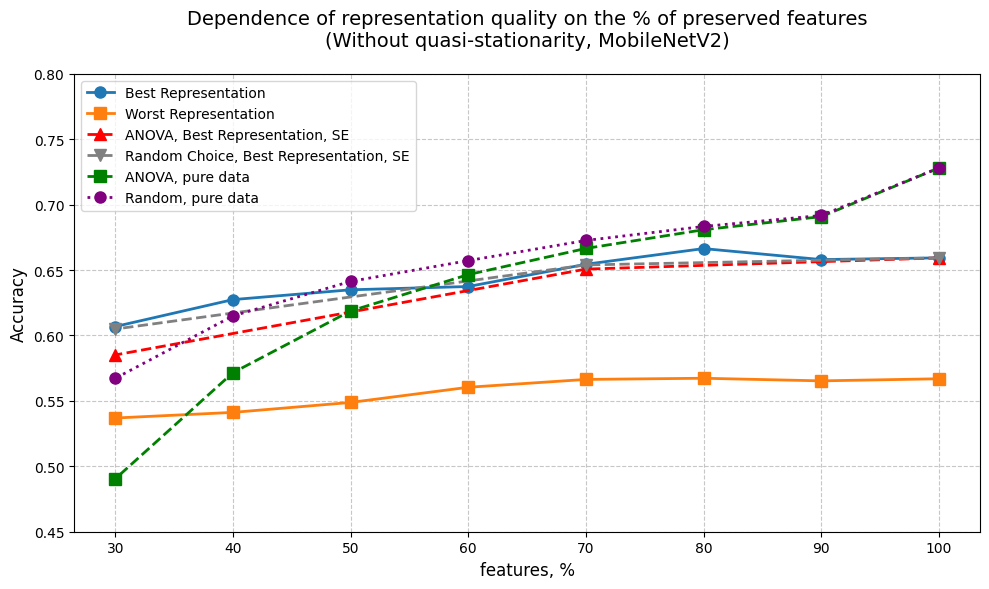}
    \caption{MobileNetV2: (left) quasi-stationary case, (right)
             non-quasi-stationary case---representation quality versus the
             percentage of preserved features.
             Solid blue line: NBSE spectral feature selector, best
             representation.  Solid orange line: NBSE spectral feature
             selector, worst representation.  Dashed red line:
             NBSE + ANOVA under Spectral Embedding (SE), best
             representation.  Dashed gray line: NBSE + random selection,
             best representation.  Dashed green line: original data +
             ANOVA.  Dotted purple line: original data + random selection.}
    \label{fig:accuracy-curves1}
\end{figure*}
\begin{figure*}[t]
    \centering
    \includegraphics[width=0.48\linewidth]{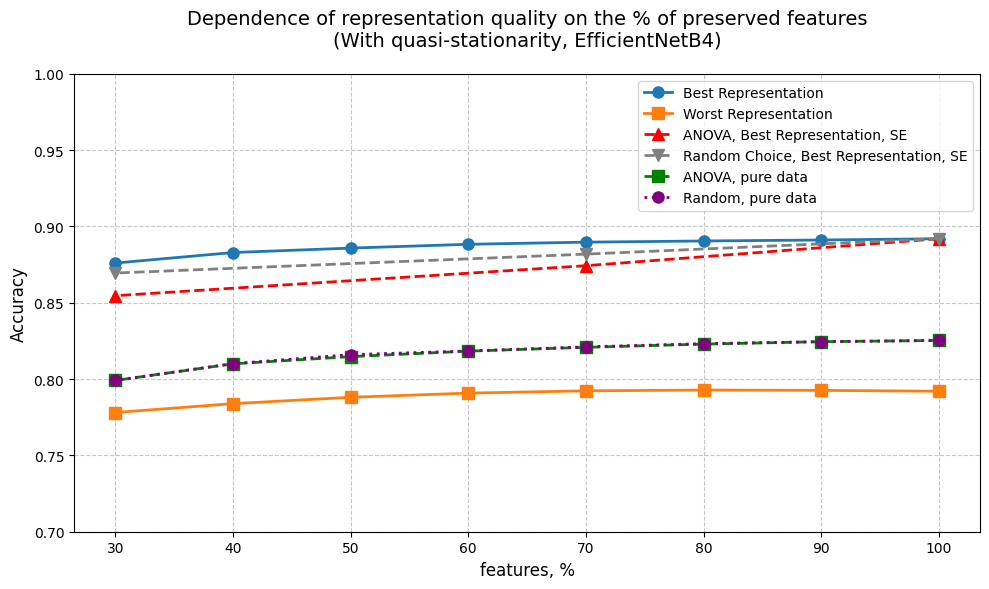}
    \hfill
    \includegraphics[width=0.48\linewidth]{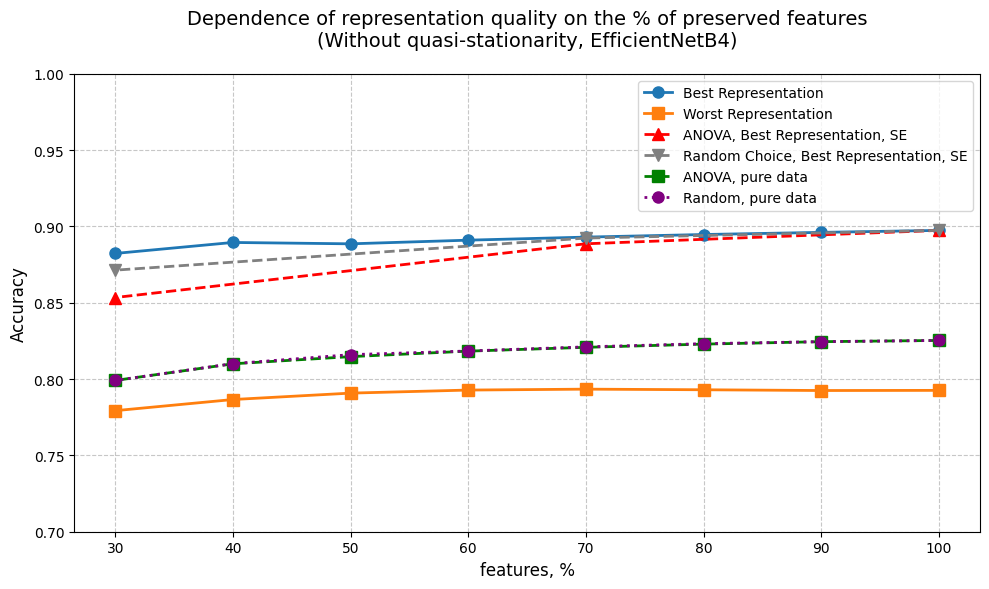}
    \caption{EfficientNet-B4: (left) quasi-stationary case---single
             $\beta_N$ computed on the full object graph and reused for all
             features; (right) non-quasi-stationary case---separate
             $\beta_N^{(l)}$ per feature slice.  The near-identical curves
             validate the quasi-stationarity hypothesis, justifying the
             simplified global computation for production use.  Line styles as
             in Fig.~\ref{fig:accuracy-curves1}.}
    \label{fig:accuracy-curves2}
\end{figure*}

\section{Discussion}
The experimental results provide compelling evidence that degree-corrected,
diffusion-based feature selection offers distinct advantages over classical
statistical filters for high-dimensional deep-learning embeddings.

\paragraph*{Global correlation structure.}
NBSE captures global correlations that univariate methods inherently miss.
Classical filter methods such as the ANOVA $F$-test operate under an
implicit independence assumption, scoring each dimension solely on its own
variance relative to class labels.  Deep-network embeddings are highly
redundant; multiple neurons often encode semantically similar information
because of over-parameterisation.  By constructing a similarity graph and
analysing the Bethe--Hessian spectrum, NBSE models the manifold structure
of the data directly.  The Nishimori identities guarantee that at $\beta_N$
thermal averages satisfy exact gauge symmetries, while the degree-corrected
diagonal $\tilde D_{ii}(\beta)$ ensures that selected features are not merely
those with high marginal variance but those contributing most significantly
to the global diffusion process.  Consequently, the intrinsic geometry of
class clusters is preserved even when $70\%$ of dimensions are discarded.

\paragraph*{Stability and representation quality.}
A striking finding is the disparity in stability between spectral embedding
space and original data space.  For MobileNetV2, aggressive ANOVA reduction
on raw features caused accuracy drops up to $25\%$, whereas NBSE maintained
fluctuations within a narrow $5\%$ band for good representations.  The
diffusion process acts as a regulariser, smoothing out high-frequency noise
that typically destabilises greedy selection algorithms.  Even with a
suboptimal data slice, the spectral method exhibited only $3\%$
fluctuation, indicating that the embedding itself provides a robust
substrate less sensitive to sampling variance than raw feature vectors.

\paragraph*{The ANOVA paradox in spectral space.}
Random selection outperformed ANOVA within the spectral embedding.  This
counter-intuitive result arises because diffusion embeddings transform
features into a coordinate system where dimensions are highly correlated by
design (reflecting shared diffusion modes).  The greedy nature of ANOVA
selects top-scoring features independently, thereby picking multiple members
of the same diffusion cluster and yielding diminishing returns.  Random
selection inadvertently samples across distinct diffusion modes.  NBSE's
spectral ablation explicitly bins the eigenvector $\boldsymbol\phi_{\min}$
and selects representatives from disjoint intervals, ensuring diversity.

\paragraph*{Architectural implications.}
The performance gap between MobileNetV2 and EfficientNet-B4 underscores how
feature selection interacts with backbone architecture.  EfficientNet-B4's
superior stability (less than $1\%$ loss at $70\%$ reduction) suggests that
deeper, more efficient networks produce denser, semantically richer
embeddings better suited for spectral analysis.  As backbone networks
improve, the efficacy of physics-guided feature selection increases.
Moreover, validation of the quasi-stationary hypothesis demonstrates
computational viability: computing a single $\beta_N$ reduces complexity
from $\mathcal O(D\cdot M\log M)$ to $\mathcal O(M\log M)$, which is
feasible for real-time applications where per-feature optimisation would be
prohibitive.

\paragraph*{Limitations and future directions.}
NBSE relies on constructing a sparse similarity graph; in extremely
high-noise regimes where manifold structure is obscured,
Proposition~\ref{prop:noise-stability} warns that $\beta_N$ may shift
appreciably.  Additionally, the current implementation focuses on binary
spectral clustering via the smallest eigenvector; extending to multi-class
Nishimori formulations could enhance performance on datasets with complex
class boundaries.  Future work will explore adaptive graph construction in
which both protograph topology and edge weights are learned directly from
data.

In summary, NBSE is a principled approach grounded in statistical physics
that aligns feature selection with the intrinsic degree-corrected diffusion
properties of the data manifold.

\section{Conclusion}
We introduced Noise-Based Spectral Embedding (NBSE), a diffusion-guided
feature-selection technique grounded in statistical physics.  By exploiting
the Nishimori temperature as an analytically determined optimal diffusion
parameter and using the full Bethe--Hessian---including its
cavity-susceptibility diagonal
$\tilde D_{ii}(\beta)=\sum_j\sinh^2(\beta W_{ij})$ and effective off-diagonal
couplings $S_{ij}(\beta)=\sinh(\beta W_{ij})\cosh(\beta W_{ij})$---to
capture the slowest heat-kernel mode, NBSE provides (a) a principled
one-dimensional spectral coordinate for each object; (b) a spectral ablation
procedure that groups redundant features without greedy heuristics;
(c) robust performance under realistic noise and aggressive dimensionality
reduction.

Proposition~\ref{prop:noise-stability} rigorously bounds the shift in the
Nishimori temperature caused by coloured Gaussian perturbations, confirming
that the method remains stable for noise levels typical of deep embeddings.
The corrected diffusion interpretation via Eq.~\eqref{eq:bh-laplacian}
shows that the degree-corrected Bethe--Hessian avoids the spectral bias of
standard Laplacians on sparse heterogeneous graphs---the regime where
classical spectral methods fail.

Experimental results confirm that the synergy of spectral data
representation and specialised selection preserves information during
significant compression.  On EfficientNet-B4, accuracy loss was limited to
less than $1\%$ even at $70\%$ feature reduction, outperforming baselines by
up to $6.8\%$.

Future work will scale NBSE to larger datasets (e.g.\ ImageNet-21K),
develop adaptive graph-construction procedures that learn both protograph
topology and edge weights from data, investigate ensembles of graph
representations, and extend the framework to multi-class Nishimori
formulations with non-binary label priors leveraging non-binary QC-LDPC
codes.



\begin{thebibliography}{99}
\bibitem{Jolliffe}
I.~T. Jolliffe,
\emph{Principal Component Analysis}, 2nd ed.
New York, NY, USA: Springer, 2002.

\bibitem{Maaten2008}
L.~van der Maaten and G.~Hinton,
``Visualizing data using t-SNE,''
\emph{J. Mach. Learn. Res.}, vol.~9, pp.~2579--2605, 2008.

\bibitem{McInnes2018}
L. McInnes and J. Healy,
``UMAP: Uniform manifold approximation and projection for dimension
reduction,''
\emph{arXiv:1802.03426}, 2018.

\bibitem{Guyon2003}
I.~Guyon and A.~Elisseeff,
``An introduction to variable and feature selection,''
\emph{J. Mach. Learn. Res.}, vol.~3, pp.~1157--1182, 2003.

\bibitem{Fisher1936}
R.~A. Fisher,
``The use of multiple measurements in taxonomic problems,''
\emph{Annals Eugenics}, vol.~7, no.~2, pp.~179--188, 1936.

\bibitem{Kohavi1997}
R.~Kohavi and G.~John,
``Wrappers for feature subset selection,''
\emph{Artif. Intell.}, vol.~97, nos.~1--2, pp.~273--324, 1997.

\bibitem{Zou2005}
H.~Zou and T.~Hastie,
``Regularization and variable selection via the elastic net,''
\emph{J. R. Stat. Soc. B}, vol.~67, no.~2, pp.~301--320, 2005.

\bibitem{Ng2002}
A.~Y. Ng, M.~Jordan, and Y.~Weiss,
``On spectral clustering: Analysis and an algorithm,''
in \emph{Advances in Neural Information Processing Systems},
vol.~14, 2002, pp.~849--856.

\bibitem{Belkin2003}
M.~Belkin and P.~Niyogi,
``Laplacian eigenmaps for dimensionality reduction and data
representation,''
\emph{Neural Comput.}, vol.~15, no.~6, pp.~1373--1396, 2003.

\bibitem{Zhu2003}
X.~Zhu, Z.~Ghahramani, and J.~Lafferty,
``Semi-supervised learning using Gaussian fields and harmonic
functions,''
in \emph{Proc. Int. Conf. Machine Learning}, 2003, pp.~912--919.

\bibitem{DallAmico2021}
L. Dall'Amico, R.~Couillet, and N.~Tremblay,
``A unified framework for spectral clustering in sparse graphs,''
\emph{J. Mach. Learn. Res.}, vol.~22, no.~217, pp.~1--56, 2021.

\bibitem{Usatyuk2024}
V.~S. Usatyuk, D.~A. Sapozhnikov, and S.~I. Egorov,
``Enhanced image clustering with random-bond Ising models using LDPC
graph representations and Nishimori temperature,''
\emph{Moscow Univ. Phys. Bull.}, vol.~79, suppl.~2,
pp.~S647--S665, 2024.

\bibitem{Usatyuk2025}
V. S. Usatyuk, D. A. Sapozhnikov, and S. I. Egorov,
``Natural image classification via quasi-cyclic graph ensembles and
random-bond Ising models at the Nishimori temperature,''
\emph{Moscow Univ. Phys. Bull.}, vol.~80, suppl.~3,
pp.~S1039--S1053, 2025.

\bibitem{Nishimori2001}
H.~Nishimori,
\emph{Statistical Physics of Spin Glasses and Information Processing:
An Introduction}.
Oxford, U.K.: Oxford Univ. Press, 2001.

\bibitem{Saade2014}
A. Saade, F. Krzakala, and L. Zdeborov\'a,
``Spectral clustering of graphs with the Bethe Hessian,''
in \emph{Advances in Neural Information Processing Systems},
vol.~27, 2014, pp.~406--414.

\bibitem{METLDPC}
T. J. Richardson and R. L. Urbanke,
``Multi-edge type LDPC codes,''
presented at the Workshop honoring Prof. Bob McEliece, Pasadena, CA, USA,
2002.

\bibitem{Anova}
F. Pedregosa et al.,
``Scikit-learn: Machine learning in Python,''
\emph{J. Mach. Learn. Res.}, vol.~12, pp.~2825--2830, 2011.

\bibitem{Sandler2018}
M.~Sandler, A.~Howard, M.~Zhu, A.~Zhmoginov, and L.-C. Chen,
``MobileNetV2: Inverted residuals and linear bottlenecks,''
in \emph{Proc. IEEE Conf. Computer Vision and Pattern Recognition},
2018, pp.~4510--4520.

\bibitem{Tan2019}
M.~Tan and Q.~Le,
``EfficientNet: Rethinking model scaling for convolutional neural
networks,''
in \emph{Proc. Int. Conf. Machine Learning}, 2019, pp.~6105--6114.
\end{thebibliography}
\end{document}